%% file: main.tex
\documentclass{article}
\usepackage[utf8]{inputenc}

\usepackage{amsmath,amsthm,amssymb}
\usepackage{bm}
\usepackage[linesnumbered,algoruled,boxed,lined]{algorithm2e}
\usepackage{color}
\usepackage{makecell}
\usepackage{hyperref}
\hypersetup{ 
    colorlinks=true,
    citecolor=blue
}

\usepackage[nonatbib, preprint]{neurips_2019}

\newtheorem{theorem}{Theorem}
\newtheorem{corollary}{Corollary}
\usepackage{xcolor}
\usepackage{booktabs}
\usepackage{appendix}
\usepackage{graphicx}
\usepackage{enumitem}
\usepackage{multirow}

\usepackage{floatrow}

\newcolumntype{C}{@{\hspace{2pt}}c}

\newcommand\blfootnote[1]{%
  \begingroup
  \renewcommand\thefootnote{}\footnote{#1}%
  \addtocounter{footnote}{-1}%
  \endgroup
}

\title{Frontal Low-rank Random Tensors for Fine-grained Action Segmentation}
\author{Yan Zhang, Krikamol Muandet, Qianli Ma, Heiko Neumann and Siyu Tang}

\author{%
  Yan Zhang$^{1,3}$ \; Krikamol Muandet$^2$\; Qianli Ma$^2$\; Heiko Neumann$^3$\; Siyu Tang$^{1}$ \\
  $^1$ETH Z\"{u}rich, Switzerland \\
  $^2$Max Planck Institute for Intelligent Systems, T\"{u}bingen, Germany \\
  $^3$Institute of Neural Information Processing, Ulm University, Germany\\
}

\date{March 2020}

\begin{document}

\maketitle

\begin{abstract}
    Fine-grained action segmentation in long untrimmed videos is an important task for many applications such as surveillance, robotics, and human-computer interaction. To understand subtle and precise actions within a long time period, second-order information (e.g. feature covariance) or higher is reported to be effective in the literature. However, extracting such high-order information is considerably non-trivial. In particular, the dimensionality increases exponentially with the information order, and hence gaining more representation power also increases the computational cost and the risk of overfitting.
    In this paper, we propose an approach to representing high-order information for temporal action segmentation via a simple yet effective bilinear form. Specifically, our contributions are: (1) From the multilinear perspective, we derive a bilinear form of low complexity, assuming that the three-way tensor has low-rank frontal slices. (2) Rather than learning the tensor entries from data, we sample the entries from different underlying distributions, and prove that the underlying distribution influences the information order. (3) We employed our bilinear form as an intermediate layer in state-of-the-art deep neural networks, enabling to represent high-order information in complex deep models effectively and efficiently. 
    Our experimental results demonstrate that the proposed bilinear form outperforms the previous state-of-the-art methods on the challenging temporal action segmentation task. One can see our project page for data, model and code: \url{https://vlg.inf.ethz.ch/projects/BilinearTCN/}.
    \blfootnote{$^*$ This work was mainly done when Y.Z. and S.T. were at MPI-IS and University of T\"{u}bingen.}
\end{abstract}

\input{intro}
\input{related_work}

\input{method}

\input{experiments}

\input{conclusion}

\bibliographystyle{unsrt} 
\bibliography{references}


\clearpage

\begingroup
\onecolumn 

\appendix
\section*{Appendix}

\setcounter{page}{1}

\section{Proof of Theorem \ref{thm:rkhs_linear}}
\label{sec:proof-thm-1}

\begin{proof}
It follows from Eq.~\eqref{eq:fusion_rank_R} and the property of an inner product of rank-one operators that
\begin{align}
    \label{eq:proof_1}
    k(\bm{z}_1, \bm{z}_2)
    &:= \langle \bm{z}_1, \bm{z}_2 \rangle \\
    &= \left\langle \sum_{r=1}^R vec(\bm{E}^r\bm{x}_1 \otimes \bm{F}^r\bm{y}_1), \sum_{r=1}^R vec(\bm{E}^r\bm{x}_2 \otimes \bm{F}^r\bm{y}_2) \right\rangle_{\mathbb{R}^{\mathit{MN}}}  \nonumber \\
    &= \sum_{r=1}^R\sum_{r'=1}^R\left\langle  vec(\bm{E}^r\bm{x}_1 \otimes \bm{F}^r\bm{y}_1), vec(\bm{E}^{r'}\bm{x}_2 \otimes \bm{F}^{r'}\bm{y}_2) \right\rangle_{\mathbb{R}^{\mathit{MN}}}  \nonumber \\
    &= \sum_{r=1}^R\sum_{r'=1}^R \left\langle  \bm{E}^r\bm{x}_1, \bm{E}^{r'}\bm{x}_2  \right\rangle 
    \left\langle \bm{F}^r\bm{y}_1, \bm{F}^{r'}\bm{y}_2 \right\rangle  \nonumber \\
    &= \sum_{r=1}^R\sum_{r'=1}^R \left( \sum_{i=1}^M \langle \bm{e}_i^r, \bm{x}_1 \rangle \langle \bm{e}_i^{r'}, \bm{x}_2 \rangle \right)\left( \sum_{j=1}^N  \langle \bm{f}_j^{r}, \bm{y}_1 \rangle \langle \bm{f}_j^{r'}, \bm{y}_2 \rangle \right).
\end{align}
Then, it follows that
\begin{eqnarray*}
    \mathbb{E}[k(\bm{z}_1,\bm{z}_2)] 
    &=& \sum_{r=1}^R\sum_{r'=1}^R \left( \sum_{i=1}^M \mathbb{E}[\langle \bm{e}_i^r, \bm{x}_1 \rangle \langle \bm{e}_i^{r'}, \bm{x}_2 \rangle] \right)\left( \sum_{j=1}^N  \mathbb{E}[\langle \bm{f}_j^{r}, \bm{y}_1 \rangle \langle \bm{f}_j^{r'}, \bm{y}_2 \rangle] \right) \\
    &=& \sum_{r=1}^R \left( \sum_{i=1}^M \mathbb{E}[\langle \bm{e}_i^r, \bm{x}_1 \rangle \langle \bm{e}_i^{r}, \bm{x}_2 \rangle] \right)\left( \sum_{j=1}^N  \mathbb{E}[\langle \bm{f}_j^{r}, \bm{y}_1 \rangle \langle \bm{f}_j^{r}, \bm{y}_2 \rangle] \right) \\
    && + \sum_{r=1}^R\sum_{r'=r+1}^R \left( \sum_{i=1}^M \mathbb{E}[\langle \bm{e}_i^r, \bm{x}_1 \rangle \langle \bm{e}_i^{r'}, \bm{x}_2 \rangle] \right)\left( \sum_{j=1}^N  \mathbb{E}[\langle \bm{f}_j^{r}, \bm{y}_1 \rangle \langle \bm{f}_j^{r'}, \bm{y}_2 \rangle] \right) \\
    &=& RMN\langle \bm{x}_1, \bm{x}_2 \rangle\langle \bm{y}_1, \bm{y}_2 \rangle .
\end{eqnarray*}
\noindent The last equation follows from \cite[Lemma 2]{kar2012random} and the fact that $\bm{e}_i^r$ and $\bm{f}_j^r$ are zero-mean random variables for all $i=1,\ldots,M$, $j=1,\ldots,N$ and $r = 1,\ldots,R$. 

\end{proof}

\section{Proof of Corollary \ref{cor:variance-bound-linear}}
\label{sec:proof-vb-linear}
\begin{proof}
    Let $W_{ij}:= \langle \bm{e}_i, \bm{x}_1 \rangle \langle \bm{e}_i, \bm{x}_2 \rangle \langle \bm{f}_j, \bm{y}_1 \rangle \langle \bm{f}_j, \bm{y}_2 \rangle$ for $i=1,\ldots,M$ and $j=1,\ldots,N$. For each $W_{ij}$, it follows from \cite[Lemma 4]{kar2012random} that
    $$-p^2f(p\tilde{R}^2)^2 \leq W_{ij} \leq p^2f(p\tilde{R}^2)^2,$$
    where we assume without loss of generality that $p\geq 1$ and $\tilde{R} \geq 1$. In our case, $f(x)=x$. Then, we have $-p^4\tilde{R}^4 \leq W_{ij} \leq p^4\tilde{R}^4$.
    Let $S_{\mathit{MN}} := \frac{1}{R}\sum_{i=1}^M\sum_{j=1}^N W_{ij}$. 
    Hence, we have $\mathbb{E}[S_{\mathit{MN}}] = MN\langle \bm{x}_1,\bm{x}_2\rangle\langle\bm{y}_1,\bm{y}_2\rangle$.
    Then, it follows from Hoeffding's inequality that, for all $\epsilon >0$, 
    \begin{equation}
        \mathbb{P}(|S_{\mathit{MN}} - \mathbb{E}[S_{\mathit{MN}}]| \geq \epsilon) \leq 2\exp\left(-\frac{\epsilon^2 MN}{2p^8\tilde{R}^8}\right).
    \end{equation}
    This concludes the proof.
\end{proof}

\section{Proof of Theorem \ref{thm:rkhs_gaussian}}
\label{sec:kernel-approximation}

\begin{proof}
Let $R=1$. Then, $k(\bm{z}_1, \bm{z}_2) := \langle \bm{z}_1, \bm{z}_2 \rangle =  \langle \varphi(\bm{E} \bm{x}_1), \varphi(\bm{E} \bm{x}_2) \rangle \langle \varphi(\bm{F} \bm{x}_1), \varphi(\bm{F} \bm{x}_2) \rangle$ where
\begin{eqnarray*}
\varphi(\bm{E}\bm{x}) &:=& \frac{1}{\sqrt{M}}[\sin(\bm{e}_1^\top\bm{x}),\ldots,\sin(\bm{e}_M^\top\bm{x}),\cos(\bm{e}_1^\top\bm{x}),\ldots,\cos(\bm{e}_M^\top\bm{x})], \\ \varphi(\bm{F}\bm{y}) &:=& \frac{1}{\sqrt{N}}[\sin(\bm{f}_1^\top\bm{y}),\ldots,\sin(\bm{f}_M^\top\bm{y}),\cos(\bm{f}_1^\top\bm{y}),\ldots,\cos(\bm{f}_M^\top\bm{y})].
\end{eqnarray*}
With $\bm{E} = \frac{1}{\sigma}\bm{I}_{M\times D_X}\bm{R}\bm{P}$ and $\bm{F} = \frac{1}{\rho}\bm{I}_{N\times D_Y}\bm{S}\bm{Q}$, we have

\begin{align*}
    \mathbb{E}[k(\bm{z}_1, \bm{z}_2)] 
    &= \mathbb{E}[\langle \varphi(\bm{E} \bm{x}_1), \varphi(\bm{E} \bm{x}_2) \rangle] \mathbb{E}[\langle \varphi(\bm{F} \bm{y}_1), \varphi(\bm{F} \bm{y}_2) \rangle] \\
    &=\exp\left(-\frac{\|\bm{x}_1-\bm{x}_2\|_2^2}{2\sigma^2}\right)\exp\left(-\frac{\|\bm{y}_1-\bm{y}_2\|_2^2} {2\rho^2}\right),
\end{align*}
\noindent where the last equality follows from \cite[Theorem 1]{yu2016orthogonal}. For $R > 1$, we have 
$$k(\bm{z}_1, \bm{z}_2) := \langle \bm{z}_1, \bm{z}_2 \rangle = \sum_{r=1}^R\sum_{r'=1}^R \langle \varphi(\bm{E}^r \bm{x}_1), \varphi(\bm{E}^{r'} \bm{x}_2) \rangle \langle \varphi(\bm{F}^r \bm{x}_1), \varphi(\bm{F}^{r'} \bm{x}_2) \rangle.$$
Hence, with $\bm{E}^r = \frac{1}{\sigma^r}\bm{I}_{M\times D_X}\bm{R}^r\bm{P}^r$ and $\bm{F}^r = \frac{1}{\rho^r}\bm{I}_{N\times D_Y}\bm{S}^r\bm{Q}^r$,
\begin{align*}
    \mathbb{E}[k(\bm{z}_1, \bm{z}_2)] 
    &=  \sum_{r=1}^R\sum_{r'=1}^R \mathbb{E}[\langle \varphi(\bm{E}^r \bm{x}_1), \varphi(\bm{E}^{r'} \bm{x}_2) \rangle]\mathbb{E}[\langle \varphi(\bm{F}^r \bm{y}_1), \varphi(\bm{F}^{r'} \bm{y}_2) \rangle] \\
    &= \sum_{r=1}^R \mathbb{E}[\langle \varphi(\bm{E}^r \bm{x}_1), \varphi(\bm{E}^{r} \bm{x}_2) \rangle]\mathbb{E}[\langle \varphi(\bm{F}^r \bm{y}_1), \varphi(\bm{F}^{r} \bm{y}_2) \rangle]  \\
    & +\sum_{r=1}^R\sum_{r'=r+1}^R \mathbb{E}[\langle \varphi(\bm{E}^r \bm{x}_1), \varphi(\bm{E}^{r'} \bm{x}_2) \rangle]\mathbb{E}[\langle \varphi(\bm{F}^r \bm{y}_1), \varphi(\bm{F}^{r'} \bm{y}_2) \rangle] \\
    &= \sum_{r=1}^R\exp\left(-\frac{\|\bm{x}_1-\bm{x}_2\|_2^2}{2\sigma^2_r}\right)\exp\left(-\frac{\|\bm{y}_1-\bm{y}_2\|_2^2} {2\rho^2_r}\right) \\
    & +\sum_{r=1}^R\sum_{r'=r+1}^R \mathbb{E}[\langle \varphi(\bm{E}^r \bm{x}_1), \varphi(\bm{E}^{r'} \bm{x}_2) \rangle]\mathbb{E}[\langle \varphi(\bm{F}^r \bm{y}_1), \varphi(\bm{F}^{r'} \bm{y}_2) \rangle] \\
    &= \sum_{r=1}^R\exp\left(-\frac{\|\bm{x}_1-\bm{x}_2\|_2^2}{2\sigma^2_r}\right)\exp\left(-\frac{\|\bm{y}_1-\bm{y}_2\|_2^2} {2\rho^2_r}\right) + b
\end{align*}
\noindent where $b := \sum_{r=1}^R\sum_{r'=r+1}^R \mathbb{E}[\langle \varphi(\bm{E}^r \bm{x}_1), \varphi(\bm{E}^{r'} \bm{x}_2) \rangle]\mathbb{E}[\langle \varphi(\bm{F}^r \bm{y}_1), \varphi(\bm{F}^{r'} \bm{y}_2) \rangle]$. 
\end{proof}

\section{Approximation error bound for Gaussian random projection}
\label{sec:proof-vb-gaussian}

We characterize the variance of $k(\bm{z}_1,\bm{z}_2)$ used in Theorem \ref{thm:rkhs_gaussian}. To simplify the presentation, we focus on the case that $R=1$.

\begin{corollary}\label{cor:variance-bound-gaussian}
Let $\bm{z}_1$, $\bm{z}_2$, and $k(\bm{z}_1,\bm{z}_2)$ be defined as in Theorem \ref{thm:rkhs_gaussian}, as well as $R=1$, $a = \|\bm{x}_1-\bm{x}_2\|_2^2/\sigma$ and $b = \|\bm{y}_1-\bm{y}_2\|_2^2/\rho$. Then, there exist functions $f$ and $g$ such that
\begin{equation}
    \mathrm{Var}(k(\bm{z}_1,\bm{z}_2)) \leq A\cdot B + A\cdot C + B\cdot D,
\end{equation}
\noindent where
\begin{align*}
    A &= \frac{1}{2M}\left[\left(\left(1-e^{-a^2}\right)^2 - \frac{M-1}{D_X}e^{-a^2}a^4\right) + \frac{f(a)}{D_X^2}\right],
    & C &= \left[ \exp\left(-\frac{b^2}{2}\right)\right]^2
    \\
    B &= \frac{1}{2N}\left[\left(\left(1-e^{-b^2}\right)^2 - \frac{N-1}{D_Y}e^{-b^2}b^4\right) + \frac{g(b)}{D_Y^2}\right], 
    & D &= \left[ \exp\left(-\frac{a^2}{2}\right)\right]^2.
\end{align*}
\end{corollary}

\begin{proof}
Let R=1. Then, we have
$$k(\bm{z}_1,\bm{z}_2) = \underbrace{\langle \varphi(\bm{E} \bm{x}_1), \varphi(\bm{E} \bm{x}_2) \rangle}_{U}
\underbrace{ \langle \varphi(\bm{F} \bm{x}_1), \varphi(\bm{F} \bm{x}_2) \rangle}_{V}.$$
Since $U$ and $V$ are independent, we have
\begin{equation}\label{eq:variance}
\mathrm{Var}(k(\bm{z}_1,\bm{z}_2)) = \mathrm{Var}(U)\mathrm{Var}(V)
+ \mathrm{Var}(U)(\mathbb{E}[V])^2 + \mathrm{Var}(V)(\mathbb{E}[U])^2,
\end{equation}
\noindent where
\begin{equation*}
\begin{aligned}[c]
    (\mathbb{E}[U])^2 &= ( \mathbb{E}\left[ \langle \varphi(\bm{E} \bm{x}_1), \varphi(\bm{E}\bm{x}_2) \rangle \right])^2 \\
    \mathrm{Var}(U) 
    &=  \mathrm{Var}\left(\langle \varphi(\bm{E} \bm{x}_1), \varphi(\bm{E} \bm{x}_2) \rangle\right)
\end{aligned}
\qquad
\begin{aligned}[c]
    (\mathbb{E}[V])^2 
    &= \left(\mathbb{E}\left[\langle \varphi(\bm{F} \bm{y}_1), \varphi(\bm{F} \bm{y}_2) \rangle \right]\right)^2 \\
    \mathrm{Var}(V) 
    &= \mathrm{Var}\left(\langle \varphi(\bm{F} \bm{y}_1), \varphi(\bm{F} \bm{y}_2) \rangle\right)
\end{aligned}
\end{equation*}
It follows from \cite[Theorem 1]{yu2016orthogonal} that 
\begin{align*}
(\mathbb{E}[U])^2 = \left( \exp\left(-\frac{\|\bm{x}_1-\bm{x}_2\|_2^2}{2\sigma^2}\right)\right)^2, \quad
(\mathbb{E}[V])^2 = \left( \exp\left(-\frac{\|\bm{y}_1-\bm{y}_2\|_2^2}{2\rho^2}\right)\right)^2.
\end{align*}
Let $a = \|\bm{x}_1-\bm{x}_2\|_2^2/\sigma$ and $b = \|\bm{y}_1-\bm{y}_2\|_2^2/\rho$. Then, by \cite[Theorem 1]{yu2016orthogonal}, there exists a function $f$ and $g$ such that
\begin{eqnarray*}
    \mathrm{Var}(U) &\leq& \frac{1}{2M}\left[\left(\left(1-e^{-a^2}\right)^2 - \frac{M-1}{D_X}e^{-a^2}a^4\right) + \frac{f(a)}{D_X^2}\right]\\
    \mathrm{Var}(V) &\leq& \frac{1}{2N}\left[\left(\left(1-e^{-b^2}\right)^2 - \frac{N-1}{D_Y}e^{-b^2}b^4\right) + \frac{g(b)}{D_Y^2}\right] .
\end{eqnarray*}
Substituting everything back into \eqref{eq:variance} yields the result.
\end{proof}
\endgroup

\section{Bilinear Residual Module: Another Modification of MS-TCN}
\label{sec:bilinear_res_module}

In Sec. \ref{sec:temporalActionSeg}, we have domenstrated how to incorporate our proposed bilinear forms into the state-of-the-art MS-TCN network. Here we introduce an alternative method, i.e. the bilinear residual module, which is illustrated in Fig.~\ref{fig:mstcn_modify}. 
First, we use bilinear pooling to extract the second-order information, and then use a regularized power normalization \cite{zhangbilinear2018} to densify the feature and use channel-wise max normalization to re-scale the feature value \cite{lea_2017_cvpr}. Afterwards, we use a convolution layer with the receptive field 25 to reduce the feature dimension to the number of classes. To prevent overfitting we use a dropout layer, and then we compute the average between the first-order information and the high-order information. 
Our bilinear residual module is applied at the end of each single stage of MS-TCN. 

\begin{figure*}[h]
    \centering
    \includegraphics[width=\linewidth]{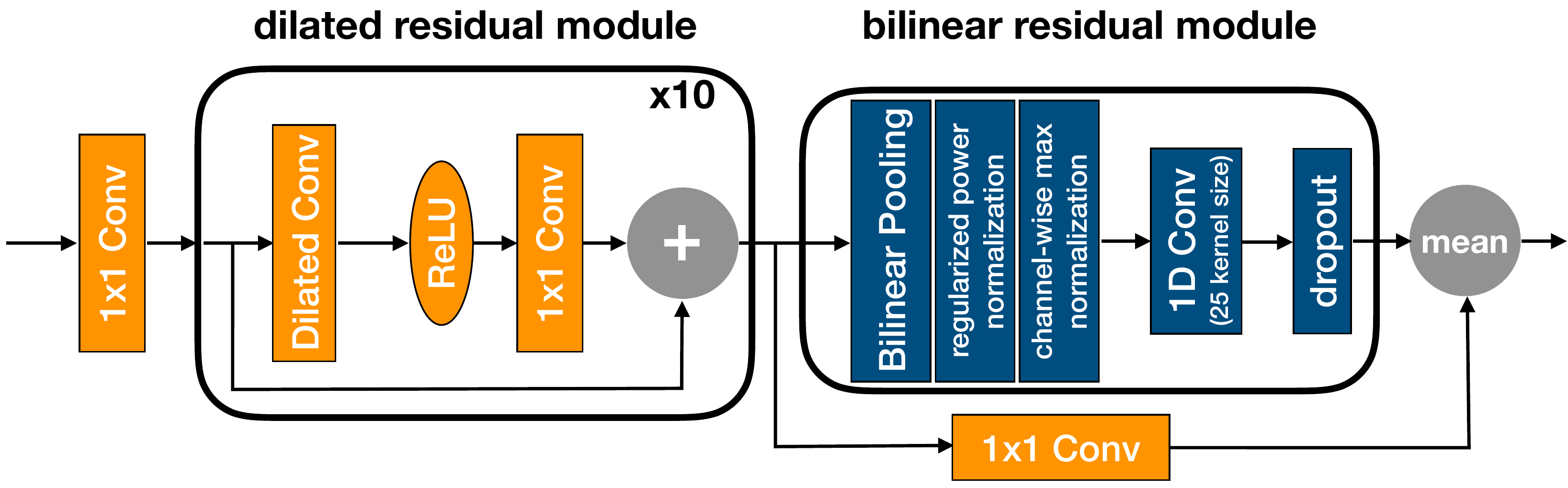}
    \caption{Illustration of an alternative approach to combine MS-TCN~\cite{farha2019ms} and bilinear pooling. In constrast to the bilinear module in Sec. \ref{sec:temporalActionSeg}, the bilinear residual module is embedded into the architecture at every stage. The orange color denotes the original layers of MS-TCN, and the blue color denotes our proposed module.}
    \label{fig:mstcn_modify}
\end{figure*}

\paragraph{Comparison with state-of-the-art.} 
Here we follow the same experimental protocol in Sec. \ref{sec:exp-mstcn}, but perform evaluation only on {\bf 50Salads} and {\bf GTEA}.  
As shown in Tab.~\ref{tab:mstcn_result}, our models (especially RPBinary and RPGaussian) consistently outperform the state-of-the-arts (TCN, TDRN, MS-TCN, LCDC), validating the effectiveness of the proposed bilinear model. Note that, even with RPGaussianFull, which is hard to train by back-propagation, our bilinear form achieves competitive performance on the {\bf GTEA} dataset.
For a fair and complete comparison, we further replace our proposed bilinear models by three widely used light-weight bilinear models in the bilinear residual module, namely, {\bf Compact}~\cite{gao2016compact}, {\bf Hadamard}~\cite{kim2016hadamard} and {\bf FBP}~\cite{li2017factorized}. Again, the proposed bilinear models show superior performance on most of the evaluation metrics. Together with the results presented in Tab.~\ref{tab:tcn_comparison}, they clearly demonstrate the representational power and the computational efficiency of the proposed bilinear models.

\begin{table*}[h]
    \centering
    \scriptsize
    \caption{Comparison with other models on temporal action segmentation task. The best results are in \textbf{boldface}. The results of MS-TCN are from \cite{farha2019ms}. }
    \begin{tabular}{lcccccccccc}
    \toprule
         &  \multicolumn{5}{c}{\bf 50 Salads} & \multicolumn{5}{c}{\bf GTEA} \\
        & Acc. & Edit & F1@0.1 &  F1@0.25 & F1@0.5 & Acc. & Edit & F1@0.1 &  F1@0.25 & F1@0.5 \\
         \cmidrule{2-6} \cmidrule{7-11}
        Ours (RPBinary) & 79.9 & 70.7& 78.0 &  {75.2}&65.4 &  {77.0} & 81.4 & {86.5} & {\bf 84.5} & {71.7} \\
        Ours (RPGaussian) & {80.6} & {\bf 71.0} & {\bf 78.4} & {\bf 75.8} &  {66.7} & {\bf 77.2} &  {82.2} &  {86.7} &  {84.3} & {\bf 72.7}\\
        Ours (RPGaussianFull) & 77.4 & 67.5 & 74.9 & 70.9 & 60.7 & 76.8 & {\bf 82.5} & {\bf 87.3} &  {84.3} &  {71.8}\\
        \midrule
        TCN~\cite{lea_2017_cvpr} & 64.7 & 59.8 & 68.0 & 63.9 & 52.6 & 64.0 & - & 72.2 & 69.3 & 56.0 \\
        TDRN~\cite{lei2018temporal} & 68.1 & 66.0 & 72.9 & 68.5 & 57.2 & 70.1 & 74.1 & 79.2 & 74.4 & 62.7\\
        MS-TCN~\cite{farha2019ms} &  {80.7} & 67.9 & 76.3 & 74.0 & 64.5 & 76.3 & 79.0& 85.8 & 83.4 & 69.8 \\
        LCDC~\cite{Mac_2019_ICCV}  &  72.1 & 66.9 & 73.8 & - & - & 65.3 & 72.8 & 75.9 & - & - \\
        \midrule
        MS-TCN + {Compact}~\cite{gao2016compact} & {\bf 81.1} &  {70.7} & 77.6 & 75.1 & {\bf 67.1} & 75.8 & 80.9 & 86.0 & 83.7 & 70.2\\
        MS-TCN + {Hadamard}~\cite{kim2016hadamard} & 78.3 & 68.3 & 75.3 & 72.4 & 62.5 & {77.0} & {81.5} & 86.0 & 83.5 & 70.7 \\ 
        MS-TCN + {FBP}~\cite{li2017factorized} & 78.5 & 70.0 & 76.5 & 73.6 & 64.6 & 75.4 & 79.3 & 84.0 & 81.8 & 69.9\\
    \bottomrule
    \end{tabular}
    \label{tab:mstcn_result}
\end{table*}

\end{document}

%% file: intro.tex
\section{Introduction}

In many applications like surgery robotics, daily activity understanding, smart user interface and others, understanding fine-grained human actions like hand-object interaction is essential.
Despite significant progress made for analyzing activities in videos with advanced deep temporal convolutional neural networks \cite{farha2019ms,lea_2017_cvpr}, detecting and segmenting fine-grained actions in long untrimmed videos remains a challenging task. 
The main difficulties lie in a differentiation of fine-grained motion patterns that are often subtle and difficult to be differentiated even by humans.

It is reported in the literature that modeling second-order feature interactions such as statistical covariance and feature co-occurrence significantly improves the performance of visual tasks like fine-grained image recognition \cite{lin2018bilinear,kong2017low}, and visual question answering \cite{Kim2018,fukui16mcb}. 
When considering higher-order information, the extracted features are more discriminative and yield impressive performance \cite{cui2017kernel, koniusz2017higher}. 
In addition, the study of \cite{zhangbilinear2018} shows that second-order information can consistently boost the performance of a temporal convolutional net \cite{lea_2017_cvpr}, brings benefits for fine-grained action segmentation. Yet, how to \emph{efficiently} extract the high-order feature interaction for the temporal action segmentation task, especially in the context of  state-of-the-art large-scale temporal convolutional neural networks \cite{farha2019ms}, has been left largely unexploited.

The critical drawback of most existing methods is that
{\em the computational cost increases exponentially with respect to the order of interactions and the input feature dimension}. 
Let us take bilinear pooling for second-order information extraction as an example. Given two feature vectors ${\bm x}\in\mathbb{R}^{D_X}$ and ${\bm y}\in\mathbb{R}^{D_Y}$, a generic bilinear model \cite[Eq. (2)]{ben2017mutan} is given by:

\begin{equation}
    \label{eq:tensor_generic}
    \bm{z} = \bm{\mathcal{T}} \times_{1} \bm{x} \times_{2} \bm{y},
\end{equation}
in which $\bm{\mathcal{T}} \in \mathbb{R}^{D_X \times D_Y \times D_Z}$ is a three-way tensor, and the operations $\times_{1}$ and $\times_{2}$ are the mode-1 and mode-2 multiplication, respectively. 
Despite being a powerful scheme, this model tends to suffer from the {\em curse of dimensionality} and {\em intractable computation} \cite{bellman2015adaptive}. 
Specifically, the number of free parameters in $\bm{\mathcal{T}}$ grows cubically with the feature dimensionality, i.e., $\mathcal{O}(D_XD_YD_Z)$, hence limiting its use in large-scale scenarios.

To reduce the complexity, a common solution is to impose specific assumptions on the structure of $\bm{\mathcal{T}}$ in Eq.~\eqref{eq:tensor_generic}. 
In the work of \cite{kim2016hadamard, ben2017mutan,kong2017low}, low-rank assumption on $\bm{\mathcal{T}}$ is employed. In the work of \cite{gao2016compact}, count sketch is used to extract approximated second-order information, which can be represented by a shorter feature vector than the full second-order information.
Although such conventional bilinear pooling methods can extract second-order information effectively, {extracting higher-order information requires extra operations \cite{cui2017kernel} and may further increase the computational cost. } The problem becomes even more drastic when we extract high-order information in the state-of-the-art temporal convolutional neural networks \cite{farha2019ms,lea_2017_cvpr}, where the bilinear pooling is employed locally as intermediate layers. As a result, the underlying deep models become 
memory-intensive and computationally intractable.

In this paper, we aim at deriving a simple bilinear model, with which changing the order of information does not involve extra computation. Therefore, our investigations have two aspects:
(1) From the multilinear perspective, we assume that each frontal slice of $\bm{\mathcal{T}}$ is a low-rank matrix. Then we derive a new bilinear form that first reduces the feature dimension and then performs vector outer product. Being different from \cite{kim2016hadamard} and \cite{ben2017mutan}, that first lift the feature dimension and then perform Hadamard product, our method has runtime complexity $\mathcal{O}(D\sqrt{d}+d)$ and space complexity $\mathcal{O}(D\sqrt{d})$ compared to $\mathcal{O}(Dd+d)$ and $\mathcal{O}(D{d})$ in their works\footnote{$D$ and $d$ are input and output feature dimensions of the bilinear model, respectively; usually, $D \ll d$.}. Thus, our bilinear operation is lightweight and can be employed in deep neural networks as intermediate layers in an efficient manner. 
(2) Rather than learning the model parameters from data, the entries of each frontal slice are determined by random projection matrices drawn from a specific distribution. Our key insight is that, while the low-rank structure allows a reduction in the number of parameters, the loss in representational power is compensated by the random projection matrices.
To demonstrate this, we show that when these random matrices are sampled with different underlying distributions, the model approximates feature maps of different reproducing kernel Hilbert spaces (RKHSs). For example, when they are sampled from the Rademacher distribution, the model approximates the multiplication of linear kernels (cf. Theorem \ref{thm:rkhs_linear}). When they are sampled from the Gaussian distribution with orthogonality constraints, the model approximates multiplication of Gaussian kernels (cf. Theorem \ref{thm:rkhs_gaussian}). Therefore, we can explicitly manipulate the model capacity without sacrificing the computational efficiency.

To apply the proposed bilinear form to the action segmentation task, we propose a novel pooling module, which can be integrated into large-scale temporal convolutional deep neural networks as an intermediate layer. The experimental results show that our method yields superior efficiency and performance to other related methods on different datasets. 

Our contributions can be summarized as follows:
\begin{itemize}
    \item We derive a novel bilinear model with a random three-way tensor. 
    Using low-rank decomposition of each frontal slice, our method significantly reduces the number of parameters of the bilinear model, and hence can serve as a computationally efficient feature fusion operation.
    \item Based on different tensor entry distributions, we prove that the proposed random tensors can estimate the feature maps to reproducing kernel Hilbert spaces (RKHSs) with different compositional kernels. Therefore, the order of information is influenced by such underlying distributions.
    \item We present new state-of-the-art results for the temporal action segmentation task.
\end{itemize}

%% file: related_work.tex
\section{Related Work}
\label{sec:related-work}

\paragraph{Fine-grained temporal action segmentation.} 
The task of detecting and segmenting actions in videos has been studied in several works \cite{singh2016multi,richard2016temporal,6751336,singh2016first,fathi2013modeling,wang2015action,Lei_2018_CVPR,Mac_2019_ICCV}.
\cite{lea2016learning} proposes a latent convolutional skip chain CRF model in the temporal domain, and learns a set of action primitive with structured support vector machine.
\cite{lea_2017_cvpr} proposes a temporal convolutional encoder-decoder network, which models fine-grained actions using a series of  temporal convolutions, pooling, and upsampling layers.
In \cite{lea2016segmental}, a spatiotemporal CNN is proposed, which contains  a spatial component and a temporal component. The objects and their relationships in video frames are captured by the spatial convolutional filters, and changes of the object relationships over time is captured by the temporal convolution operation.
\cite{farha2019ms} proposes a multi-stage temporal convolutional network, which uses consecutive dilated convolutions to increase the receptive field over time. A temporal smoothing loss is also proposed to overcome the over-segmentation issue.
To capture high-order feature interaction, we propose a bilinear module and integrate it into the work of \cite{farha2019ms}. \cite{zhangbilinear2018} is closely related to our work. In \cite{zhangbilinear2018}, the second-order feature interaction is modelled by the proposed bilinear form, which has a low-dimensional representation without extra computation. Our work has the clear advantages that we can model higher order feature interaction without increasing the computation cost. We also show that our model can be employed in complex deep models whereas \cite{zhangbilinear2018} cannot.

\paragraph{Representing high-order information.} 
Due to the computational cost, second-order information is mostly considered in various practical tasks, such as fine-grained image understanding~\cite{carreira2012semantic, gao2016compact, yu2018statistically, koniusz2017higher, li2017second, lin2015bilinear,lin2018bilinear, lin2018second, kong2017low,li2018towards, wei2018grassmann,gou2018monet}, fine-grained action understanding~\cite{feichtenhofer16,girdhar2017attentional,cherian2017higher,zhangbilinear2018}, visual question answering~\cite{yu2018beyond,ben2017mutan,Kim2018,fukui16mcb} and beyond. 
In many studies, second-order information is extracted only once before the classification layer in a deep neural net~\cite{yu2018statistically, lin2015bilinear, lin2018bilinear, lin2018second, kong2017low, wang2017g2denet, diba2017deep,feichtenhofer16}. For example, \cite{lin2018bilinear} computes the vector outer product to merge outputs from two individual CNN streams. \cite{wei2018grassmann} transforms CNN features to the compact Grassmann manifold using singular-value decomposition. These methods either lift the feature dimension from $D$ to $D^2$, or introduce expensive computation in both forward and backward passes, and hence are not applicable as intermediate layers in deep neural networks.
Consequently, extracting higher-order ($\geq 3$) information is even more challenging, although it is reported that higher-order information is more discriminative. \cite{koniusz2017higher} computes higher-order feature co-occurrence in the {\em bag-of-words} pipeline, rather than in deep neural networks.

\paragraph{Low-rank tensor structure.} 
There exist many investigations on low-rank tensor decomposition for bilinear models. 
For instance, \cite{kim2016hadamard} assumes that each frontal slice of the three-way tensor can be decomposed into two low-rank matrices, and the fusion of the two input features can then be achieved by matrix multiplication and Hadamard product. 
To improve the performance of visual question answering, \cite{yu2017mfb} introduces more operations after the low-rank bilinear pooling \cite{kim2016hadamard}, such as dropout, power normalization, and L2 normalization. 
\cite{ben2017mutan} relies on Tucker decomposition of the tensor, producing three matrices and a smaller core three-way tensor. 

\paragraph{High-order information approximation.} 
Representing high-order information can explicitly lead to combinatorial explosions. 
To avoid the {\em curse of dimensionality}, feature approximation methods based on {\em reproducing kernel Hilbert space} (RKHS) theories have been recently investigated.
For example, \cite{kar2012random} uses binary random entries to approximate inner product kernels, especially the $p$-th order polynomial kernels. \cite{gao2016compact} and \cite{pham2013fast} use tensor sketch to approximate polynomial kernels, which have lower approximation error bound but higher computational cost. \cite{yu2016orthogonal} uses orthogonal random features to approximate feature maps of Gaussian kernels. 
To boost the computational speed, a structured version with normalized Walsh-Hadamard matrices is proposed. Such feature approximation methods also improve the efficiency of higher-order information extraction. For example, based on count sketch \cite{gao2016compact}, \cite{cui2017kernel} extracts $4^{\text{th}}$-order information from CNN features in an efficient manner.

Herein, we design a new bilinear form with lower complexity than \cite{kim2016hadamard}, \cite{yu2017mfb} and \cite{ben2017mutan}. Also, we provide a generic framework to approximate feature maps of RKHSs with compositional kernels. Via sampling tensor entries from different distributions, we show that the output feature vector lies within certain RKHS, and hence we can manipulate the model capacity while retaining the same computational complexity.

%% file: method.tex
\section{Method}\label{sec:method}

In this section, we first introduce how to decrease the number of parameters in the bilinear model Eq.~\eqref{eq:tensor_generic} via low-rank assumption. Afterwards, we present our investigations on how to approximate feature maps of kernels via random projection. 
In the end, we present our bilinear module, which incorporates high order information into the state-of-the-art temporal convolutional networks.

\subsection{Tensor Frontal Low-rank Approximation}
Here we follow the tensor notations in \cite{kolda2009tensor}.
Eq.~\eqref{eq:tensor_generic} can be re-written in terms of matrix-vector multiplication as
\begin{equation}
    \label{eq:tensor_unfolding}
    \bm{z} = \bm{T_{(3)}}\, vec(\bm{x} \otimes \bm{y}),
\end{equation}
\noindent where $\bm{T_{(3)}}$ is the mode-3 matricization of $\bm{\mathcal{T}}$, $vec(\cdot)$ denotes column-wise vectorization of a matrix, and $\otimes$ denotes vector outer product. 
In other words, the bilinear operation in Eq.~\eqref{eq:tensor_generic} is equivalent to first computing the correlation matrix between the two features and then performing a linear projection.

It follows from Eq.~\eqref{eq:tensor_unfolding} that each entry of the output feature vector $\bm{z}$ is a weighted sum of all the entries in the correlation matrix $\bm{x} \otimes \bm{y}$, i.e.,
\begin{eqnarray}
    \label{eq:each_z_entry}
    z_k &=& \langle vec(\mathcal{T}[:,:,k]), vec(\bm{x} \otimes \bm{y})\rangle \nonumber \\ 
    &=& \sum_{i=1}^{D_X}\sum_{j=1}^{D_Y} \mathcal{T}[i,j,k]v_{f(i,j)},
\end{eqnarray}
where ${\bm v} := vec(\bm{x} \otimes \bm{y})$. If the frontal matrix $\mathcal{T}[:,:,k]$ is a rank-one matrix, i.e., $\mathcal{T}[:,:,k] = \bm{e} \otimes \bm{f}$ for some $\bm{e} \in \mathbb{R}^{D_X}$ and $\bm{f} \in \mathbb{R}^{D_Y}$, then we can rewrite~Eq. \eqref{eq:each_z_entry} as $$z_k 
    = \langle vec(\bm{e} \otimes \bm{f}), vec(\bm{x} \otimes \bm{y}) \rangle 
    = \langle \bm{e}, \bm{x} \rangle \langle \bm{f}, \bm{y} \rangle .$$
It is important to note that it is the second equality that allows us to reduce the computational cost when constructing the fused feature $\bm{z}$.

Thus, we define the projection matrices $\bm{E} = [\bm{e}_1,\bm{e}_2,...,\bm{e}_M]^T$ and $\bm{F} = [\bm{f}_1,\bm{f}_2,...,\bm{f}_N]^T$ for two sets of vectors $\{\bm{e}_i\}_{i=1}^M \subset{\mathcal{X}}$ and $\{\bm{f}_j\}_{j=1}^N \subset{\mathcal{Y}}$ with $M \leq D_X$ and $N \leq D_Y$. 
Then, the fusion map $\phi:\mathbb{R}^{D_X}\times\mathbb{R}^{D_Y}\to\mathbb{R}^{MN}$ can be defined as 
\begin{equation}
    \label{eq:fusion_rank_1}
    \bm{z} := \phi (\bm{x},\bm{y}) = vec\left( (\bm{E} \bm{x}) \otimes (\bm{F} \bm{y}) \right).
\end{equation}

If we assume further that $\mathcal{T}[:,:,k]$ is a rank-$R$ matrix, i.e., $\mathcal{T}[:,:,k] = \sum_{r=1}^R \bm{e}^r_i \otimes \bm{f}^r_j$, we obtain 
\begin{equation}
    \label{eq:fusion_rank_R}
    \bm{z} := \phi (\bm{x},\bm{y}) = vec\left( \sum_{r=1}^R (\bm{E}^r \bm{x}) \otimes (\bm{F}^r \bm{y}) \right),
\end{equation}
where $\bm{E}^r = [\bm{e}_1^r,\bm{e}_2^r,...,\bm{e}_M^r]^T$ and $\bm{F}^r = [\bm{f}_1^r,\bm{f}_2^r,...,\bm{f}_N^r]^T$ for $r=1,2,...,R$. 
With such a low-rank assumption, we avoid computing the high-dimensional correlation matrix $\bm{x}\otimes\bm{y}$, which considerably reduces the model parameters from $D_XD_YD_Z$ to $R(MD_X+ND_Y)$ with a small value of $R$.

\paragraph{Our novelties.}
A similar low-rank assumption is also used in \cite{ben2017mutan} and \cite{kim2016hadamard}, in which the two input feature vectors are first projected to a common vector space and then fused via the Hadamard product. Assuming the input feature vectors are of the same dimension $D$ and the output feature vector is of dimension $d$, it requires $\mathcal{O}(Dd+d)$ operations to compute and requires $\mathcal{O}(Dd)$ memory to store. In contrast, our method requires $\mathcal{O}(D\sqrt{d}+d)$ for computation and $\mathcal{O}(D\sqrt{d})$ for storage. Since in practice it normally needs more dimensions to represent a higher-order feature, i.e., $d \gg D$, our method has a consistently better runtime (see Tab. \ref{tab:tcn_comparison}), and hence is more suitable to be employed in a sophisticated deep neural network.

One notes that our low-rank approximation method is fundamentally different from previous work e.g. \cite{ben2017mutan,kim2016hadamard}. 
First, the vector outer product in Eq. \eqref{eq:tensor_unfolding} is directly derived from the generic bilinear model {\em without low-rank assumption}, while the outer product in previous work, e.g. \cite{ben2017mutan,kim2016hadamard}, is derived from the low-rank assumption. Consequently, our model explicitly incorporates feature channel correlations of input vectors, while previous methods do not. Second, After applying the tensor low-rank assumption, our derived bilinear models, e.g. Eq. \eqref{eq:fusion_rank_1} or Eq. \eqref{eq:fusion_rank_R}, still {\em retain the vector outer product}, while bilinear models in \cite{ben2017mutan,kim2016hadamard} lead to vector Hadamard product. Our model is conceptually different, rather than a reformulation of previous work. This means that our proposed bilinear models cannot be reformulated from previous work, and vice versa.

\subsection{Random Projection}
\label{sec:random_projection}
To compensate for the loss in model capacity caused by the low-rank assumption, we observe that the entries and associated distributions of $\bm{E}$ and $\bm{F}$ defined in Eq.~\eqref{eq:fusion_rank_1} and Eq.~\eqref{eq:fusion_rank_R} can influence the model capacity, without adding or removing learnable parameters or network layers.

Inspired by this observation, we propose to manipulate the model capacity by randomly sampling the entries of $\bm{E}$ and $\bm{F}$ from specific distributions and then perform random projection. 
Unlike an end-to-end learning via back-propagation, our approach provides an alternative way of building expressive representation that is explainable and is simple to use in practice.

\paragraph{Rademacher random projection.}
Motivated by \cite{kar2012random} and \cite{gao2016compact}, we specify model parameters, i.e., the projection matrices $\bm{E}^r$ and $\bm{F}^r$ in the bilinear model \eqref{eq:fusion_rank_1} or \eqref{eq:fusion_rank_R}, with random samples from the Rademacher distribution. 
Below we show that the bilinear model given in Eq. \eqref{eq:fusion_rank_1} \emph{unbiasedly} approximates, with high probability, a feature map to a reproducing kernel Hilbert space (RKHS), in which the associated kernel is the multiplication of two linear kernels in $\mathcal{X}$ and $\mathcal{Y}$, respectively.

\begin{theorem}
\label{thm:rkhs_linear}
    Let $\bm{E}^r \in \mathbb{R}^{M\times D_X}$ and $\bm{F}^r \in \mathbb{R}^{N\times D_Y}$ for any $r\in\{1,2,...,R\}$ be Rademacher random matrices whose entries are determined by an independent Rademacher random variable $\sigma\in\{-1,1\}$. 
    For any $\bm{x}_1,\bm{x}_2 \in \mathcal{X}$ and $\bm{y}_1,\bm{y}_2 \in \mathcal{Y}$, let $\bm{z}_1 = \phi(\bm{x}_1,\bm{y}_1)$ and $\bm{z}_2=\phi(\bm{x}_2,\bm{y}_2)$ be the output features given by Eq. \eqref{eq:fusion_rank_1}. 
    Define a kernel function by $k(\bm{z}_1,\bm{z}_2) =\langle \bm{z}_1,\bm{z}_2\rangle$, then we have 
    \begin{align}
    \nonumber \mathbb{E}[k(\bm{z}_1, \bm{z}_2)]&=RMN\langle \bm{x}_1, \bm{x}_2 \rangle\langle \bm{y}_1, \bm{y}_2 \rangle.
    \end{align}
\end{theorem}

Next, we characterize the error of such kernel approximations.

\begin{corollary}\label{cor:variance-bound-linear}
Let $\bm{z}_1$ and $\bm{z}_2$ be defined as in Theorem \ref{thm:rkhs_linear}. 
Let $k(\bm{z}_1,\bm{z}_2) = \frac{1}{R}\langle \bm{z}_1,\bm{z}_2\rangle$. 
Then, the following inequality holds:
\begin{equation}
    \mathbb{P}\left(|k(\bm{z}_1,\bm{z}_2) - \mathbb{E}[k(\bm{z}_1,\bm{z}_2)]| > \epsilon\right) \leq 2 \exp\left(-\frac{\epsilon^2 MN}{2p^8\tilde{R}^8}\right),
\end{equation}
\noindent for some constants $\epsilon > 0$, and $p\geq 1$, $\tilde{R} \geq 1$, which are independent of the feature dimensions. 
\end{corollary}

Proofs of both results can be found in Appendix \ref{sec:proof-thm-1} and \ref{sec:proof-vb-linear}. 
More details on the constants $p$ and $\tilde{R}$ can be found in \cite{kar2012random}.
To remove the effect of the scaling factors, we rewrite Eq. \eqref{eq:fusion_rank_R} as 
\begin{equation}
    \label{eq:binary_normalized}
    \bm{z} = \phi(\bm{x},\bm{y})= \frac{1}{R\sqrt{MN}} \cdot vec\left(\sum_{r=1}^R (\bm{E}^r \bm{x} ) \otimes (\bm{F}^r \bm{y} ) \right).
\end{equation}

Therefore, Eq.~\eqref{eq:binary_normalized} with {\it binary tensor entries} is capable of capturing second-order interactions between two features. We refer this bilinear form as ``RPBinary'' in the experiment section.

\paragraph{Gaussian random projection.}
To increase the model capacity, a common approach is to apply the nonlinear activation function $\varphi(\cdot)$, which gives
\begin{equation}
    \label{eq:fusion_rank_R_nonlinear}
    \bm{z} := \phi (\bm{x},\bm{y}) = {vec}\left( \sum_{r=1}^R \varphi(\bm{E}^r \bm{x}) \otimes \varphi(\bm{F}^r \bm{y}) \right).
\end{equation}

\noindent Inspired by \cite{yu2016orthogonal}, we consider
\begin{align}\label{eq:orthogonal-proj}
&\bm{E}^r = \frac{1}{\sigma^r}\bm{I}_{M\times D_X}\bm{R}^r\bm{P}^r, \quad 
\bm{F}^r = \frac{1}{\rho^r}\bm{I}_{N\times D_Y}\bm{S}^r\bm{Q}^r \quad \nonumber \\
&\text{with} \quad r=1,2,...,R,
\end{align}
\noindent where $\bm{R}^r$ and $\bm{S}^r$ are diagonal matrices with diagonal entries sampled i.i.d. from the chi-squared distributions $\chi^2(D_X)$ and $\chi^2(D_Y)$ with $D_X$ and $D_Y$ degrees-of-freedom, respectively, $\bm{P}^r$ and $\bm{Q}^r$ are uniformly distributed random orthogonal matrices\footnote{Specifically, $\bm{P}^r$ and $\bm{Q}^r$ are uniformly distributed on the Stiefel manifold \cite{yu2016orthogonal,muirhead2009aspects}.}, and $\bm{I}_{M\times D_X}$ and $\bm{I}_{N\times D_Y}$ are identity matrices with the first $M$ and $N$ rows, respectively.
Here, $\{\sigma^r\}_{r=1}^R$ and $\{\rho^r\}_{r=1}^R$ are tunable bandwidth parameters, and can be learned from data.

When the nonlinear function in Eq. \eqref{eq:fusion_rank_R_nonlinear} is given by $\varphi(\bm{E}\bm{x}) := \sqrt{1/M}[\sin(\bm{E}\bm{x}),\cos(\bm{E}\bm{x})]$ and $\varphi(\bm{F}\bm{y}) := \sqrt{1/N}[\sin(\bm{F}\bm{y}),\cos(\bm{F}\bm{y})]$, the resulting representation approximates feature maps to the RKHSs corresponding to a composition of Gaussian kernels (see Appendix \ref{sec:kernel-approximation} for the proof).

\begin{theorem}
\label{thm:rkhs_gaussian}
Let $\bm{E}^r \in \mathbb{R}^{M\times D_X}$ and $\bm{F}^r \in \mathbb{R}^{N\times D_Y}$ for any $r\in\{1,2,...,R\}$ be random matrices whose entries are determined as in Eq.~\eqref{eq:orthogonal-proj}. 
For any $\bm{x}_1,\bm{x}_2 \in \mathcal{X}$ and $\bm{y}_1,\bm{y}_2 \in \mathcal{Y}$, let $\bm{z}_1 = \phi(\bm{x}_1,\bm{y}_1)$ and $\bm{z}_2=\phi(\bm{x}_2,\bm{y}_2)$ be the output features in Eq.~\eqref{eq:fusion_rank_R_nonlinear}. 
Define a kernel function $k(\bm{z}_1,\bm{z}_2) =\langle \bm{z}_1,\bm{z}_2\rangle$, then we have 
    \begin{align}
        \mathbb{E}[k(\bm{z}_1, \bm{z}_2)] \nonumber 
        &= \sum_{r=1}^R\exp\left(-\frac{\|\bm{x}_1-\bm{x}_2\|_2^2}{2\sigma^2_r}\right)\exp\left(-\frac{\|\bm{y}_1-\bm{y}_2\|_2^2} {2\rho^2_r}\right) \nonumber \\
        & + \sum_{r=1}^R\sum_{r'=r+1}^R  \Big\{\mathbb{E}[\langle \varphi(\bm{E}^r \bm{x}_1), \varphi(\bm{E}^{r'} \bm{x}_2) \rangle] \nonumber \\
        & \qquad\qquad\qquad \times\mathbb{E}[\langle \varphi(\bm{F}^r \bm{y}_1), \varphi(\bm{F}^{r'} \bm{y}_2) \rangle] \Big\}.
    \end{align}
\end{theorem}
\normalsize

A characterization of the variance of $k(\bm{z}_1,\bm{z}_2)$ in Theorem \ref{thm:rkhs_gaussian} is given in Appendix \ref{sec:proof-vb-gaussian},
which suggests that higher output feature dimension can reduce the variance of $k(\bm{z}_1,\bm{z}_2)$. 
Consequently, just by using the periodic function $\varphi(\cdot)$ and the Gaussian random projection, we can obtain infinite-order information without extra computational cost. We refer to this form as ``RPGaussianFull'' in the experiment section. In principal, one can extract different types of high-order information by using different nonlinear functions in Eq.~\eqref{eq:fusion_rank_R_nonlinear}. Further investigations on such problems are currently beyond our scope.

In practice, when used as an intermediate layer in deep neural networks, $\sin$ and $\cos$ functions are known to be difficult and unreliable to train using back-propagation \cite{parascandolo2016taming}. 
Therefore, we perform Taylor expansion of $\sin$ and $\cos$, and only use the first term. Namely, we approximate $\sin(\bm{E}x) \approx \bm{Ex}$ and $\cos(\bm{E}x) \approx 1$.  Then, we can discard the nonlinear function $\varphi(\cdot)$ in Eq. \eqref{eq:fusion_rank_R_nonlinear}, and also can apply scaling factors as in Eq. ~\eqref{eq:binary_normalized}. We refer this approximated version as ``RPGaussian'' in our experiments.
In addition, we adopt a simpler version $\bm{E}^r = \frac{\sqrt{D_X}}{\sigma^r}\bm{I_{M\times D_X}}\bm{P}^r$ and $\bm{F}^r = \frac{\sqrt{D_Y}}{\rho^r}\bm{I_{N\times D_Y}}\bm{Q}^r$. 
According to \cite{yu2016orthogonal}, such a simplified version exhibits similar empirical behavior to the original version, especially when the feature dimensionality is high. 
Moreover, rather than regarding the Gaussian radii as hyper-parameters, we learn them via back-propagation when employing the model \eqref{eq:binary_normalized} as an intermediate layer in a deep neural network. 

It is instructive to note that, in addition to what we have proposed, the distributions of $\bm{E}$ and $\bm{F}$ can be arbitrary.
We can even model these distributions using deep generative models, which is the subject of our future work.

\paragraph{Our novelties.}
Previous studies like \cite{kar2012random,yu2016orthogonal} approximate explicit RKHS embeddings of a polynomial/Gaussian kernel. Namely, they aim to derive a more powerful representation $\phi({\bm x})$ for only one {\em single} input vector ${\bm x}$. In contrast, our method approximates explicit RKHS embeddings of \textit{compositional kernels}. 
Namely, we aim to derive a more powerful and explainable representation $\phi({\bm x}, {\bm y})$ for the \textit{interaction} of {\em two} input vectors ${\bm x}$ and ${\bm y}$. In particular, our Gaussian random projection method can capture higher-order information while retaining a low computational cost, compared to methods like \cite{gao2016compact} for capturing second-order information. Due to these fundamental differences, to guarantee that our new methods work and to explain why they work, we prove new theorems as above, in which certain steps are based on well-established theorems in previous studies. 

The advantages of random projection over learning the model parameters from data are two-fold: (1) Our random projection method is solidly backed by RKHS theories, while learning-based model typically lacks intepretability. (2) Our experiment shows that our random projection method empirically outperforms learnable models. One can compare Ours (LearnableProjection) with other random projection methods in the first part of Tab. \ref{tab:tcn_comparison}.

\subsection{Temporal Action Segmentation}
\label{sec:temporalActionSeg}
Recent action segmentation approaches rely on temporal convolution neural network, in order to capture long-range temporal correlations between video frames \cite{farha2019ms,lea_2017_cvpr}. 
To boost performance with minimal architecture modification, here we demonstrate how to combine our bilinear forms with two state-of-the-art deep neural networks for action segmentation, namely the temproal convolutional network (TCN) \cite{lea_2017_cvpr} and the multi-stage temporal convolution network (MS-TCN) \cite{farha2019ms}. These two methods take video frame-wise features as input, and produce frame-wise action labels as output.

According to \cite{lea_2017_cvpr}, TCN comprises two encoder blocks and two decoder blocks, each of which has symmetric architectures with the encoder. In the TCN architecture, we replace the original max pooling layers in the two encoders by our bilinear pooling methods, like in \cite{zhangbilinear2018}.

According to \cite{farha2019ms}, MS-TCN does not use an encoder-decoder architecture, but heavily employs dilated temporal convolutions. As indicated by its name, MS-TCN incorporates multiple stages. Except the first stage which takes frame-wise features as input, other stages take the output from previous stages as input, which is the frame-wise action class probability. In this case, we only introduce our bilinear forms into the first stage to extract higher-order frame-wise features, as higher-order information in other stages lack straightforward meanings.

Specifically, after consecutive dilated temporal convolution layers in the first stage, we replace the last 1x1 convolution layer by a bilinear pooling module, which incorporates a bilinear pooling layer to extract high-order information, a temporal convolution layer to convert feature dimensions to the number of action classes, and a dropout layer to avoid overfitting. Since the high-order information tends to partition an action into smaller segments \cite[Fig. 1]{zhangbilinear2018}, we use a convolution receptive field 25 according to \cite{lea_2017_cvpr}. We find that higher and lower values than 25 leads to inferior performance. The influence of the \textit{dropout in the bilinear module} is investigated in Sec. \ref{sec:exp-mstcn}. Additionally, it is noted that neither nonlinear activation function nor normalization after the bilinear pooling operation is employed .
The modified first stage is illustrated in Fig. \ref{fig:mstcn_modify2} (b).

\begin{figure*}[h]
    \centering
    \includegraphics[width=\linewidth]{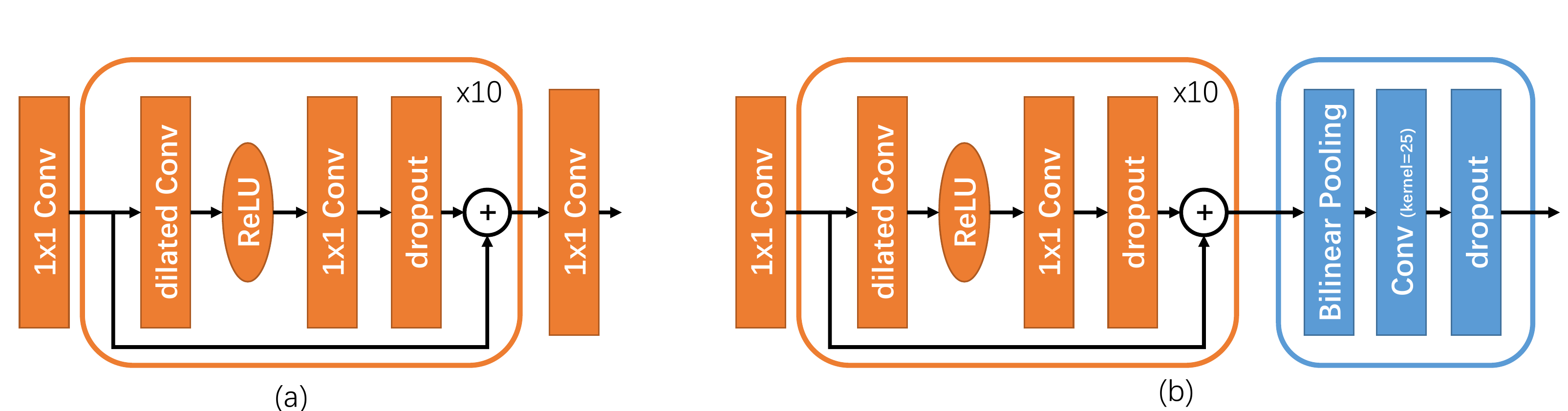}
    \caption{Illustration of combining MS-TCN~\cite{farha2019ms} and our bilinear pooling method. (a) The original architecture of a single stage in MS-TCN, which is also used for $\geq 2$ stages in our modified version of MS-TCN. (b) The original last 1x1 convolution layer is replaced by our bilinear module, which is only used in the first stage of our modified MS-TCN. The orange color denotes the original layers of MS-TCN, and the blue color denotes our proposed bilinear module.}
    \label{fig:mstcn_modify2}
\end{figure*}

%% file: experiments.tex
\section{Experiment}\label{sec:experiments}

We conduct experiments for fine-grained temporal action segmentation, which aims at assigning each individual frame an action label. In such experiments, the two input features $\bm{x}$ and $\bm{y}$ in Eq. \eqref{eq:binary_normalized} and Eq. \eqref{eq:fusion_rank_R_nonlinear} are identical, and we set
the same number of rows (i.e. $M=N$) to matrices $\bm{E}^r$ and $\bm{F}^r$ for all $r=1,\ldots,R$. Consequently, there only remain two hyper-parameters in the bilinear model, i.e. the rank $R$ and the number of rows $N$ of matrices.
The objectives of our experiments are two-fold: (i) To verify the effectiveness of our method, we adopt the temporal convolutional network (TCN)~\cite{lea_2017_cvpr} due to its simple structure, and replace the max pooling in TCN by bilinear pooling as in \cite{zhangbilinear2018}. (ii) We propose a bilinear module (see Sec. \ref{sec:temporalActionSeg}), and combine it with the multi-stage temporal convolutional net (MS-TCN) \cite{farha2019ms} to yield state-of-the-art performance.

\paragraph{Datasets.} In our experiments, we use the {\bf 50Salads} dataset, \cite{stein2013combining}, the {\bf GTEA} dataset \cite{fathi2011learning,li2015delving}, and the {\bf Breakfast} dataset \cite{Kuehne12}. 

The {\bf 50Salads} dataset comprises 50 recordings from 25 subjects, and the length of each video ranges from 5-10 minutes with the framerate of 30fps. Since we focus on fine-grained action analysis, we use its ``mid-level'' annotation, which contains 17 different actions annotated on individual frames. For evaluation, {\bf 50Salads} is split to 5 folds.
The {\bf GTEA} dataset contains long-term video recordings from the egocentric view. Specifically, the videos are recorded as 15fps with the resolution of 1280$\times$720, and have 31222 frames in total. The annotation has 11 action classes, and is also performed frame-wisely. For evaluation, {\bf GTEA} is split to 4 folds.
The {\bf Breakfast} dataset is large-scale, and contains 1712 breakfast preparation videos performed by 52 subjects in 18 different kitchens. The number of frames is over 4 million, and the number of action classes is 48. On average, each video contains 6 actions. For evaluation, {\bf Breakfast} is split to 4 folds. 

For fair comparison, we use the same video frame features as in \cite{lea_2017_cvpr, zhangbilinear2018} in the experiments about TCN \cite{lea_2017_cvpr}, and use the same video frame features as in \cite{farha2019ms} in the experiments about MS-TCN \cite{farha2019ms}.

\paragraph{Evaluation metric.} We use three standard metrics, i.e. frame-wise accuracy, edit score and F1 score as in \cite{lea_2017_cvpr}, \cite{zhangbilinear2018} and \cite{farha2019ms}.
For the F1 score, we consider the intersection-over-union (IoU) ratio of 0.1, 0.25 and 0.5, and denote them as F1@0.1, F1@0.25 and F1@0.5, respectively. Without explicit mentioning, our F1 score means F1@0.1.
Detailed definitions of these metrics are explained in \cite{lea_2017_cvpr} and \cite{farha2019ms}.
Since each dataset has several splits, we report the average results of cross-validation as in other work.

\paragraph{Notations.} 
``LearnableProjection'' indicates $\bm{E}$ and $\bm{F}$ in Eq. \eqref{eq:binary_normalized} are learned via back-propagation. ``RPBinary'' indicates the model Eq. \eqref{eq:binary_normalized} with Rademacher random projection. ``RPGaussianFull'' and ``RPGaussian'' indicate the model Eq. \eqref{eq:fusion_rank_R_nonlinear} with and without $\varphi(\cdot)$, respectively, in which Gaussian random projection is employed.

\subsection{Action Segmentation with TCN}
We use the default architecture of TCN \cite{lea_2017_cvpr}, which comprises an encoder and a decoder with symmetric modules. The convolutional layer in each individual encoder has 64 and 96 filters, respectively.
We train the model using the Adam optimizer \cite{kingma2014adam} with a fixed learning rate of $10^{-4}$. Batch size is set to 8, and the training process terminates after 300 epochs.

\paragraph{Model analysis: comparison between bilinear forms.}
Here we set the rank $R=1$ and $N={D}/2$ for our methods where $D$ is the input feature dimension to our bilinear model (see Eq. \ref{eq:fusion_rank_R_nonlinear} and Eq. \ref{eq:binary_normalized}). The results are shown in the first part of Tab. \ref{tab:tcn_comparison}.
For performance evaluation, we repeat the experiment 3 times, and report the result with the highest sum of the three metrics. 
To evaluate the efficiency, we report the runtime per batch (batch size=8) which is the averaged result after training with the first split for 300 epochs for each dataset. 
Overall, the results suggest that ``RPGaussian'' and ``RPBinary'' have comparable performances, and outperform ``LearnableProjection'' and ``RPGaussianFull''. 
Also, we suspect that the inferior performance of ``RPGaussianFull'' is caused by the periodic activation function, making the network hard to train. 

In addition, we compare our methods with three widely used bilinear pooling methods, i.e., {\bf Compact}~\cite{gao2016compact}, {\bf Hadamard}~\cite{kim2016hadamard} and {\bf FBP}~\cite{li2017factorized} and one second-order temporal bilinear pooling method ({\bf TBP}~\cite{zhangbilinear2018}). We use the same output feature dimension for fair comparison. The results are shown in the second part of Tab.~\ref{tab:tcn_comparison}. 
They suggest that these three methods perform inferior or comparably in terms of the three metrics of action parsing, and are clearly less efficient, than our methods.

\begin{table*}[t!]
    \centering
    \small
    \caption{Comparison with different bilinear pooling methods in terms of {\em accuracy/edit score/F1 score} and runtime (millisecond). 
    For each metric and each setting, the best result is in boldface. In the column of complexity, $D$ and $d$ denote the input and output feature dimension of our bilinear model, respectively.}
    \begin{tabular}{llcccc}
    \toprule
     & &\multicolumn{2}{c}{\bf 50Salads} & \multicolumn{2}{c}{\bf GTEA} \\
      \cmidrule(lr){3-4} \cmidrule(lr){5-6}
    Method & Complexity  & Runtime & Performance & Runtime & Performance \\
    \midrule
    Ours (RPBinary) & $\mathcal{O}(D\sqrt{d}+d)$& {\bf 68.3} & 66.0/{\bf 65.9}/70.9 & 80.1 & 65.2/73.2/77.0 \\
    Ours (RPGaussian) & $\mathcal{O}(D\sqrt{d}+d)$& {68.9} & 67.6/65.2/{\bf 72.9} & {\bf 69.9} & {\bf 66.9}/{76.5}/{\bf 79.8} \\
    Ours (RPGaussianFull) & $\mathcal{O}(D\sqrt{d}+d)$ & 73.7 & 64.1/63.4/69.6& 70.0 & 64.5/73.0/78.7 \\
    Ours (LearnableProjection) & $\mathcal{O}(D\sqrt{d}+d)$& 78.6 & 66.4/65.0/70.5 & 81.6 & 64.8/74.0/77.5\\
    \midrule
    {Compact}~\cite{gao2016compact} & $\mathcal{O}(D+d\log{d})$ &95.9 & 67.2/65.8/71.7 & 120.2 & 65.9/75.3/78.1\\
    {Hadamard}~\cite{kim2016hadamard} & $\mathcal{O}(Dd+d)$ & 83.8 & {\bf 67.7}/64.4/71.5 &83.6 &  66.0/{\bf 76.6}/79.0 \\
    {FBP}~\cite{li2017factorized} & $\mathcal{O}((2k+1)Dd+d)$ & 130.5 & 64.0/61.0/67.5 & 127.6 & 63.4/71.6/74.1\\
    {TBP}~\cite{zhangbilinear2018} & $\mathcal{O}(\frac{(D+1)D}{2})$ & 553.0 & 65.7/62.8/69.0 & 618.6 & 64.4/73.9/76.3\\
    \bottomrule
    \end{tabular}
    \label{tab:tcn_comparison}
\end{table*}

\paragraph{Model analysis: investigation on the hyper-parameters of our bilinear forms.} 
Here we investigate the influence of the rank and the output feature dimension of RPGaussian and RPBinary.
Fig. \ref{fig:ablation_ap} shows the dependence of the model performance on ranks $R\in\{1,2,4,8,16\}$ and matrix rows $N \in \{1,2,4,8\} \times [\sqrt{D}]$, where $[n]$ denotes the nearest integer to $n$. 
In this case, the output feature dimension is then $\{1,4,16,64\} \times D$. In all plots, the performance increases consistently with the matrix row $N$ (hence the output feature dimension). This result is in line with our theoretical analysis on the kernel approximation error bound (see Corollary \ref{cor:variance-bound-linear} and Appendix \ref{sec:proof-vb-gaussian}): Larger values of $M$ and $N$ can yield lower variance upper bounds, hence better kernel approximation.
One can also observe that the performance saturates when further increasing the output feature dimension. 
This is due to the limited capacity of the corresponding RKHS. 

Moreover, one can observe that increasing the rank may not consistently yield better performance. 
A probable reason is that an optimal rank depends on the dataset and the applied bilinear model.

\begin{figure*}[ht]
    \centering
    \footnotesize
    \includegraphics[width=\linewidth]{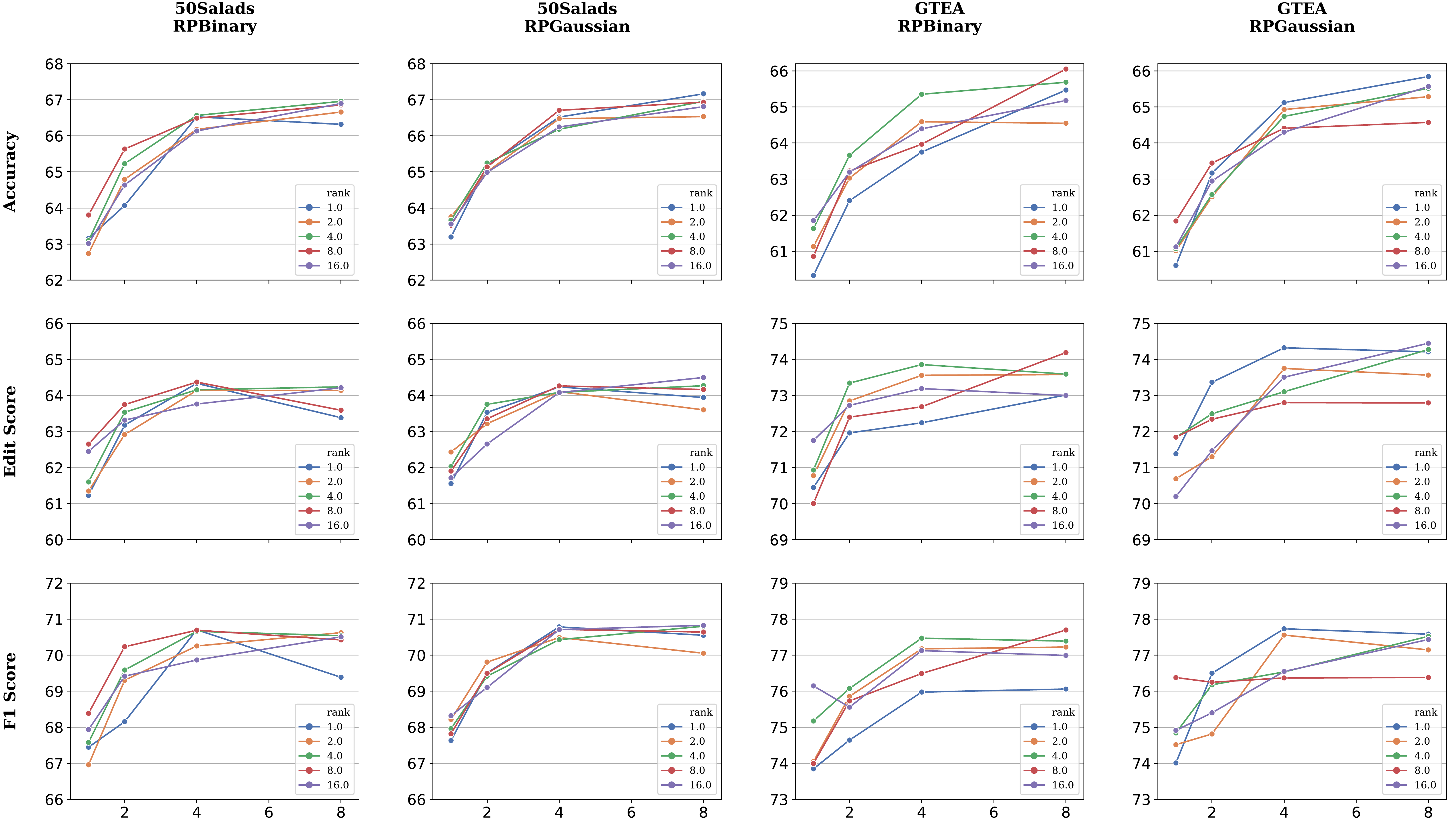}
    \caption{Performance of our RPBinary and RPGaussian model, versus dimension / rank, on datasets {\bf 50Salads} and {\bf GTEA}. In each plot, x-axis is the multiplier on the number of matrix rows $N$, y-axis is the respective performance measure, and colors denote different ranks.}
    \label{fig:ablation_ap}
\end{figure*}

\subsection{Action Segmentation with MS-TCN} 
\label{sec:exp-mstcn}
Here we verify the effectiveness of our method proposed in Sec. \ref{sec:temporalActionSeg}. 

\paragraph{Implementation details.}
Based on our previous experimental results, we set rank $R=4$ and $N=D/2$ for our bilinear model. As described in Sec.~\ref{sec:temporalActionSeg}, we use a bilinear module to extract high-order information in the end of the first stage. As presented in \cite{farha2019ms}, each single stage first uses a 1$\times$1 convolutional layer to modify the feature dimension, and then employs a dilated residual module containing 10 dilated temporal convolution operations with increasing receptive fields. We use the default network setting of MS-TCN with 4 stages, and our bilinear module is only used in the first stage. 
To conduce fair comparison with the baseline MS-TCN model, we keep other model configurations and the loss function (including the hyper-parameters) unchanged. 
Similarly to \cite{farha2019ms}, we use the Adam \cite{kingma2014adam} optimizer with a learning rate of 0.0005. The batch size is set to 1. The modified MS-TCN is trained from scratch, and training terminates after 50 epochs.

\paragraph{Results: influence of the dropout ratio.}
As presented in Sec. \ref{sec:temporalActionSeg}, the bilinear pooling module comprises a bilinear pooling model, a temporal convolution layer and a dropout layer. Here we investigate the influence of its dropout ratio, based on the {\bf GTEA} and {\bf 50Salads} datasets. For both RPGaussian and RPBinary, we only vary the dropout ratio while retaining other hyper-parameters unchanged. The results are shown in Tab. \ref{tab:dropout_exp}, in which a 0 dropout ratio means the dropout layer is not employed. 
One can see that the dropout ratio obviously influences the segmentation performance. In particular, with the Gaussian random tensors, the dropout ratio has consistent influence. The best performance is consistently achieved when the dropout ratio is 0.25. In addition, one can see that the best performance with RPGaussian is superior to the best performance with RPBinary.

\begin{table}[h]
    \centering
    \small
    \caption{The influence of the dropout ratio in the bilinear module. The best scores are highlighted in boldface.}
    \begin{tabular}{lcccccc}
        \toprule
        &   & \multicolumn{5}{c}{\bf GTEA} \\
        \textbf{Methods} & \textbf{Dropout Ratio}  & acc. & edit & F1@0.1 & F1@0.25 & F1@0.5 \\
        \midrule
        \multirow{4}{*}{RPBinary} & 0 & 76.7              & \textbf{83.5} & \textbf{87.2} & 84.2          & \textbf{73.3} \\
                                  & 0.25 & \textbf{77.4} & 82.8          & 86.6          & \textbf{84.6} & 71.3 \\
                                  & 0.5 & 76.6           & 83.0          & 87.0          & 84.5          & 72.3 \\
                                  & 0.75 & 74.0          & 81.7          & 85.8          & 81.7          & 70.2 \\
        \midrule
        \multirow{4}{*}{RPGaussian} & 0      & 77.7          & 84.0          & 88.4          & 86.2          & 73.4 \\
                                  & 0.25    & \textbf{78.5} & 84.0          & \textbf{88.5} & \textbf{86.8} & \textbf{74.6} \\
                                  & 0.5     & 76.8          & \textbf{84.6} & 87.2          & 84.3          & 73.7 \\
                                  & 0.75    & 77.7          & 80.8          & 86.0          & 84.0          & 70.8 \\
        \bottomrule
        \toprule
        &   & \multicolumn{5}{c}{\bf 50Salads} \\
        \textbf{Methods} & \textbf{Dropout Ratio}  & acc. & edit & F1@0.1 & F1@0.25 & F1@0.5 \\
        \midrule
        \multirow{4}{*}{RPBinary} & 0 & \textbf{80.3} & 67.7          & 76.2          & 73.4          & \textbf{64.7} \\
                                  & 0.25 & 75.9      & 65.0          & 72.6          & 69.9          & 61.5 \\
                                  & 0.5 & 79.4       & \textbf{68.3} & \textbf{76.4} & \textbf{73.6} & 63.9 \\
                                  & 0.75 & 70.1      & 62.4          & 68.0          & 64.8          & 54.5 \\
        \midrule
        \multirow{4}{*}{RPGaussian} & 0      & 79.2          & 67.7          & 75.5          & 72.0          & 61.9 \\
                                  & 0.25    & \textbf{80.1} & \textbf{72.3} & \textbf{78.3} & \textbf{75.0} & \textbf{66.4} \\
                                  & 0.5     & 79.4          & 70.0          & 77.3          & 74.4          & 65.4 \\
                                  & 0.75    & 70.1          & 62.5          & 68.1          & 63.7          & 53.5 \\
        \bottomrule
    \end{tabular}
    \label{tab:dropout_exp}
\end{table}

Moreover, one may think that this dropout ratio is related to temporal smoothness. For example, a higher dropout ratio could overcome over-segmentation. However, in our trials, we find that there is no consistent relation between dropout ratio and temporal regularity. In some cases, a higher dropout ratio even leads to over-segmentation. Some visualization results are shown in Fig. \ref{fig:dropout}.

\begin{figure}[h]
    \centering
    \includegraphics[width=\linewidth]{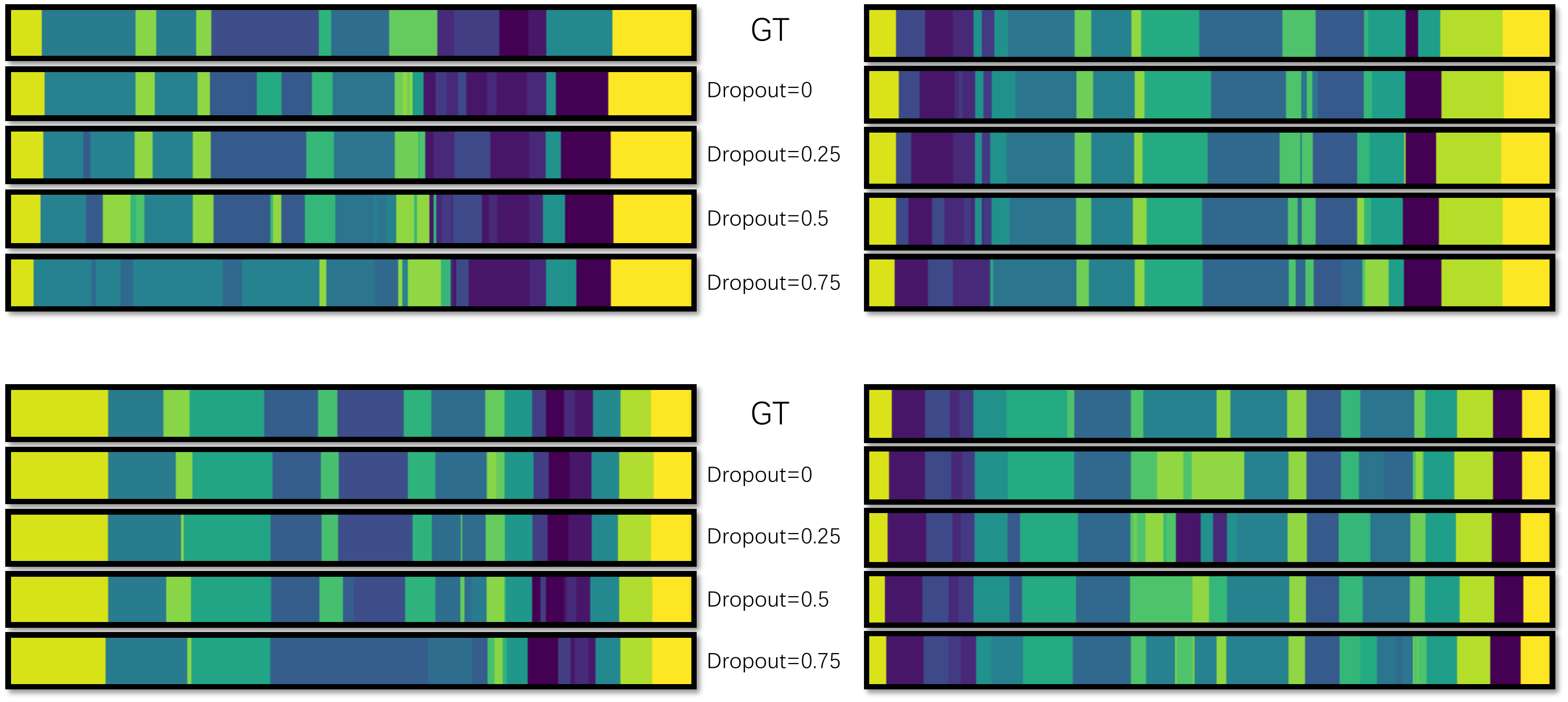}
    \caption{Visualization of influence of the dropout ratio. We randomly choose 4 test videos in {\bf 50Salads}, and show their segmentation results with RPGaussian. Different action classes are denoted by different colors. }
    \label{fig:dropout}
\end{figure}

\paragraph{Results: comparison with state-of-the-art.}
This comparison is based on the {\bf GTEA}, {\bf 50Salads}, {\bf Breakfast} datasets. According to our investigation on dropout, we choose RPBinary with no dropout, and RPGaussian(Full) with dropout ratio 0.25 as our methods to compare with other state-of-the-art.
It is noted that the original MS-TCN only uses first-order information.
The results are shown in Tab. \ref{tab:mstcn_result_alldataset} \footnote{For the method of MS-TCN \cite{farha2019ms}, we use the officially provided code and data, and report the results from our own trials. We find that our results are different from reported results in \cite{farha2019ms}. These differences are probably caused by different computation infrastructures. In our experiments for Sec. \ref{sec:exp-mstcn}, we use the CPU of AMD Ryzen Threadripper 2920X 12-Core Processor, 32G RAM, Nvidia TITAN RTX GPU with the driver version 440.59, cuda version 10.2 and CUDNN version 7602. We use PyTorch 1.2.0. }.

\begin{table*}[h]
    \centering
    \small
    \caption{Comparison with other action segmentation methods on the {\bf GTEA}, {\bf 50Salads} and {\bf Breakfast} datasets. The best results are in boldface.}
    \begin{tabular}{lccccc}
    \toprule
        &  \multicolumn{5}{c}{\bf GTEA} \\
       \textbf{Methods} & Acc. & Edit & F1@0.1 &  F1@0.25 & F1@0.5 \\
        \midrule
        Ours (RPBinary, no dropout)         & 76.7          & 83.5 & 87.2 & 84.2 & 73.3 \\
        Ours (RPGaussian, 0.25 dropout)     & \textbf{78.5} & \textbf{84.0} & \textbf{88.5} & \textbf{86.8} & \textbf{74.6}\\
        Ours (RPGaussianFull, 0.25 dropout) & 77.3          & 81.9 & 85.7 &  82.9 & 71.1\\
        \midrule
        MS-TCN~\cite{farha2019ms}           &  76.1 & 82.1 & 86.6 & 84.2 & 70.4 \\
        TCN~\cite{lea_2017_cvpr}            & 64.0 & - & 72.2 & 69.3 & 56.0 \\
        TDRN~\cite{lei2018temporal}         &  70.1 & 74.1 & 79.2 & 74.4 & 62.7\\
        LCDC~\cite{Mac_2019_ICCV}           & 65.3 & 72.8 & 75.9 & - & - \\
    \bottomrule
    \toprule
        &  \multicolumn{5}{c}{\bf 50Salads} \\
       \textbf{Methods} & Acc. & Edit & F1@0.1 &  F1@0.25 & F1@0.5 \\
        \midrule
        Ours (RPBinary, no dropout)         & \textbf{80.3} & 67.7 & 76.2 & 73.4 & 64.7 \\
        Ours (RPGaussian, 0.25 dropout)     &  {80.1}       & \textbf{72.3} & \textbf{78.3} & \textbf{75.0} & \textbf{66.4}\\
        Ours (RPGaussianFull, 0.25 dropout) & 76.9         & 66.9 & 72.9 &  69.9 & 60.1\\
        \midrule
        MS-TCN~\cite{farha2019ms}           & 79.2 & 67.1 & 74.1 & 71.8 & 62.6 \\
        TCN~\cite{lea_2017_cvpr}            & 64.7 & 59.8 & 68.0 & 63.9 & 52.6 \\
        TDRN~\cite{lei2018temporal}         & 68.1 & 66.0 & 72.9 & 68.5 & 57.2\\
        LCDC~\cite{Mac_2019_ICCV}           & 72.1 & 66.9 & 73.8 & - & - \\
    \bottomrule
    \toprule
        &  \multicolumn{5}{c}{\bf Breakfast} \\
       \textbf{Methods} & Acc. & Edit & F1@0.1 &  F1@0.25 & F1@0.5 \\
        \midrule
        Ours (RPBinary, no dropout)         & 61.4 & 56.8 & 54.1 & 48.2 & 35.8 \\
        Ours (RPGaussian, 0.25 dropout)     &  \textbf{64.2} & \textbf{63.5} & \textbf{62.0} & \textbf{56.0} & \textbf{43.7}\\
        Ours (RPGaussianFull, 0.25 dropout) & 31.5  & 36.6 & 34.0 & 27.4 & 15.4\\
        \midrule
        MS-TCN~\cite{farha2019ms}           & 63.7 & 60.0 & 49.1 & 44.4 & 34.5 \\
        TCN~\cite{lea_2017_cvpr}            & 43.3 & - & - & - & - \\
        HTK~\cite{kuehne2017weakly}         & 50.7 & - & - & - & -\\
        HTK(64)~\cite{kuehne2016end}        & 56.3 & - & - & - & -\\
        TCFPN~\cite{ding2018weakly}         & 52.0 & - & - & - & - \\
        GRU~\cite{richard2017weakly}        & 60.6 & - & - & - & - \\
    \bottomrule
    \end{tabular}
    \label{tab:mstcn_result_alldataset}
\end{table*}

%% file: conclusion.tex
\section{Conclusion}
\label{sec:conclusion}
In this work, we propose a novel bilinear model for extracting high-order information from features. To reduce the number of model parameters, we utilize low-rank tensor decomposition. Instead of using element-wise product as in other works, we use the outer product of the features to model the high-order correlations between feature channels, and hence reduce the computation complexity. 
To enrich the model representiveness while retaining the number of parameters, we use random projection to approximate feature maps to reproducing kernel Hilbert spaces associated with kernel compositions. To validate the effectiveness of our method, we perform extensive experiments on the action segmentation task, and have achieved state-of-the-art performance on challenging benchmarks. Our bilinear pooling operation is lightweight, easy to use, and can serve as a natural tool for fine-grained visual understanding and information fusion. 

In the future we will investigate how to combine our work with the deep generative models. That is, instead of sampling entries from a pre-defined distribution, we aim to learn to model these distributions using deep generative models.

\subsubsection*{Acknowledgement}
Y.~Zhang, Q.~Ma and S.~Tang acknowledge funding by Deutsche Forschungsgemeinschaft (DFG, German Research Foundation) Projektnummer 276693517 SFB 1233. 